\documentclass[runningheads]{llncs}
\usepackage{graphicx}
\usepackage{amsmath,amssymb,amsfonts}
\usepackage{bm}
\usepackage{booktabs}
\usepackage{multirow}
\usepackage{array}
\usepackage{mathtools}
\usepackage{xspace}
\usepackage{caption}
\usepackage{subcaption}
\usepackage{hyperref}
\usepackage[table]{xcolor}
\usepackage{microtype}
\usepackage{wrapfig}
\newcommand{\method}{GaugeDefect\xspace}
\newcommand{\R}{\mathbb{R}}

\newcommand{\D}{\mathcal{D}}

\newcommand{\argmin}{\mathop{\mathrm{arg\,min}}}

\newcommand{\eps}{\varepsilon}

\begin{document}
\title{GaugeDefect: Detecting Surface Anomalies by Curvature of Feature Transport}
\titlerunning{GaugeDefect}
%
\author{
	Yefan Wang\thanks{Corresponding author.}
}

\authorrunning{Y. Wang}

\institute{
	University of Shanghai for Science and Technology, Shanghai, China\\
	\email{yefanwang88@gmail.com}
}
\maketitle              
\begin{abstract}
	Industrial anomaly localization has advanced rapidly with feature-based, reconstruction-based, and distillation-based methods. Most of these methods score a region by asking how unusual its local appearance or feature representation is with respect to normal training images. This is a strong and practical formulation. In this work, we study a complementary geometric cue for cases where an abnormal region may still contain locally plausible visual features. Thin scratches, small dents, and disrupted repeated patterns often do not make every local patch individually abnormal; instead, they disturb how nearby features vary and connect across the surface. We propose \method, a geometric method for surface anomaly localization based on the curvature of feature transport. Given a feature lattice, we estimate a local feature frame at each node and compute orthogonal transports between neighboring frames. The accumulated transport around a small closed loop gives a holonomy matrix, whose deviation from identity measures feature-transport curvature. After calibration on normal training images, unusually large curvature indicates a local inconsistency in the feature field. The curvature here is not the physical curvature of the inspected object, but a representation-space measure of neighborhood inconsistency. This makes the method applicable to curved surfaces, textured materials, and non-planar industrial objects. Its main role is to improve localization of subtle surface disruptions, while often producing sharper responses near defect boundaries as a natural consequence of the curvature signal.
\end{abstract}
\begin{figure*}[t]
	\centering
	\includegraphics[width=\linewidth]{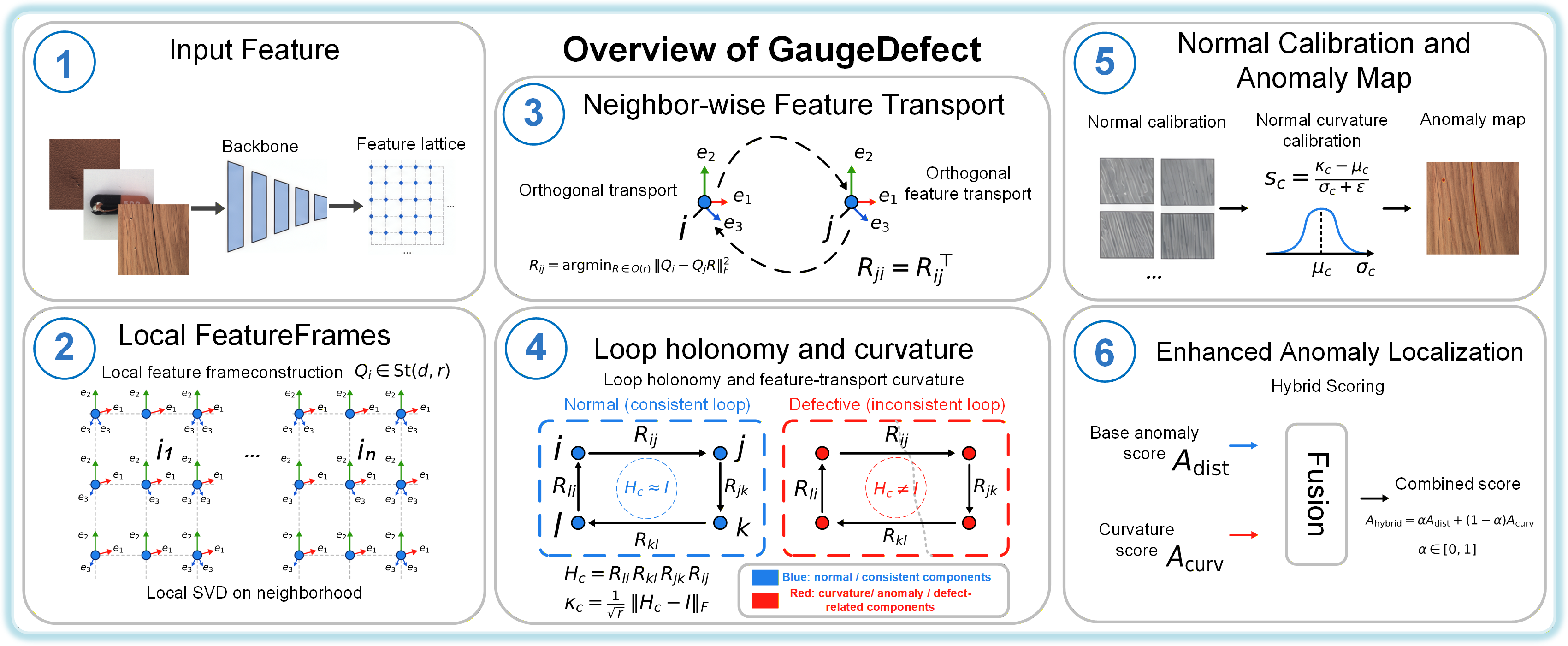}
	\caption{Overview of \method. The method extracts a feature lattice, builds local feature frames, estimates orthogonal transport, computes loop holonomy and curvature, calibrates curvature with normal data, and produces a dense anomaly map.}
	\label{fig:overview}
	\vspace{-3mm}
\end{figure*}

\section{Introduction}

Industrial visual inspection is commonly formulated as an anomaly localization problem with only normal images available during training. At test time, the system must localize defects such as scratches, dents, cracks, stains, missing parts, and subtle texture disruptions. Because abnormal samples are rare and diverse, many effective methods rely on pretrained features, reconstruction errors, distillation discrepancies, memory banks, or synthetic anomalies to estimate whether a local region is abnormal.

This line of work is effective and has led to strong anomaly localization systems. Still, not every defect is best described as a clearly out-of-distribution patch. In many cases, a defective patch is visibly different from normal patches, and a local feature distance or reconstruction discrepancy is enough to reveal it. Yet not all surface defects behave in this way. A thin scratch on a brushed metal surface may share high-frequency texture statistics with its surroundings. A small dent may only slightly disturb a smooth appearance pattern. A broken repeated pattern may still contain elements that are individually normal. In these cases, the local patch itself may remain plausible, while the relation between neighboring patches becomes inconsistent.

This observation motivates \method. Instead of replacing existing anomaly detectors, we add a geometric signal that measures local consistency. The key question is simple: if local feature frames are transported around a small closed loop, do they come back consistently to the starting point? When the loop encloses a subtle defect, the accumulated transport becomes less consistent, and this inconsistency can be measured as curvature.

Fig.~\ref{fig:overview} shows the full pipeline. An input image is first converted into a feature lattice. Each node of the lattice corresponds to a deterministic spatial unit: a patch token in a Vision Transformer feature map or a feature cell in a convolutional feature map. Around each node, we estimate a local feature frame using the features in a small spatial neighborhood. Neighboring frames are aligned by orthogonal transport, and the product of transports around an elementary loop gives a holonomy matrix. The deviation of this holonomy from identity is used as a curvature score. Finally, curvature statistics estimated from normal training images are used to calibrate the anomaly map. Fig.~\ref{fig:principle} illustrates the central idea: in a normal region, loop transport is nearly consistent; in a defective region, the same loop becomes inconsistent, leading to a larger curvature score.

A cable, bottle, screw, or textured material may naturally have non-planar geometry or nontrivial normal appearance variation. What we measure is curvature in representation space: the failure of feature transport to remain locally consistent. This method can also be applied to curved and textured materials, without requiring the underlying physical surface to be flat.

The main contributions are: (1) We formulate surface anomaly localization through feature-transport curvature, targeting defects that are locally plausible but relationally inconsistent. (2) We construct a gauge-invariant pipeline with feature-lattice nodes, local frames, orthogonal transports, loop holonomy, and normal curvature calibration. (3) We validate \method on MVTec AD, VisA, and Real-IAD, with ablations on backbone choice, node construction, frame rank, loop window, and calibration.

\section{Related Work}

\subsection{Industrial Anomaly Benchmarks}

MVTec AD is the standard benchmark for unsupervised industrial anomaly detection and localization, with pixel-level annotations across both object and texture categories~\cite{Bergmann2019MVTecA}. VisA broadens this setting with more industrial objects and defect types~\cite{Zou2022SPottheDifferenceSP}. Real-IAD moves closer to real-world inspection by providing a large-scale multi-view industrial anomaly benchmark~\cite{Wang2024RealIADAR}. Beyond appearance anomalies, MVTec LOCO AD studies logical anomalies, showing that industrial anomaly detection is not always a purely local texture problem~\cite{Bergmann2022BeyondDA}.

These datasets make it clear that the field has moved beyond simple defect spotting. A practical method should handle both obvious anomalies and subtle ones. Our work is motivated by the latter.

\subsection{Recent Anomaly Localization Methods}

PatchCore~\cite{Roth2021TowardsTR} remains one of the most influential memory-based baselines. It scores each test feature by nearest-neighbor distance to a normal feature bank. This pointwise perspective is a natural reference for our work, since \method is designed to complement rather than dismiss it.

More recent methods improve anomaly localization from different directions. SimpleNet~\cite{Liu2023SimpleNetAS} uses a simple feature-space pipeline with synthetic anomalies. EfficientAD~\cite{Batzner2023EfficientADAV} emphasizes high accuracy at low latency. RealNet~\cite{Zhang2024RealNetAF} improves synthetic anomalies through feature selection and more realistic generation. GLASS~\cite{Chen2024AUA} proposes a unified anomaly synthesis strategy. GeneralAD~\cite{Strater2024GeneralADAD} studies anomaly detection across domains by attending to distorted features. INP-Former~\cite{Luo2025ExploringIN} explores intrinsic normal prototypes within a single image. Dinomaly~\cite{Guo2024DinomalyTL} revisits multi-class unsupervised anomaly detection with a simpler design.

These methods are strong baselines. Most of them still rely mainly on local abnormality, reconstruction discrepancy, or prototype mismatch. Our method adds another signal: whether the local feature field is geometrically consistent under transport.

\subsection{Visual Feature Lattices}

Modern anomaly localization methods commonly rely on pretrained visual features. Recent visual learning increasingly exploits relational and cross-modal knowledge~\cite{10.1145/3503161.3548238,9746752,10.1145/3805622.3810732}. In this work, we use two types of feature lattices. The first is a transformer-based lattice from DINOv2~\cite{Oquab2023DINOv2LR}, where each node corresponds to a patch token at a fixed spatial location. The second is a convolutional lattice from ResNet50~\cite{He2015DeepRL}, where each node corresponds to a spatial cell in an intermediate feature map. These two choices allow us to test whether feature-transport curvature depends on a particular backbone family.

\subsection{Feature Transport and Discrete Curvature}

The mathematical construction of \method is related to classical ideas in local frame alignment and connection geometry. The transport between neighboring frames is estimated by the orthogonal Procrustes problem~\cite{Schnemann1966AGS}. The local frames lie on the Stiefel manifold and the transport matrices lie on the orthogonal group, both standard objects in matrix manifold optimization~\cite{Absil2007OptimizationAO}. The broader idea of comparing local vector spaces through a connection is related to vector diffusion maps and the connection Laplacian~\cite{Singer2011VectorDM}.

Our use of this geometry is deliberately local and practical. We do not estimate a physical surface, depth map, or global manifold. Instead, we compute discrete loop holonomy directly on the feature lattice. This makes the method easy to apply on top of standard visual backbones.

\begin{figure}[t]
	\centering
	\includegraphics[width=0.95\textwidth]{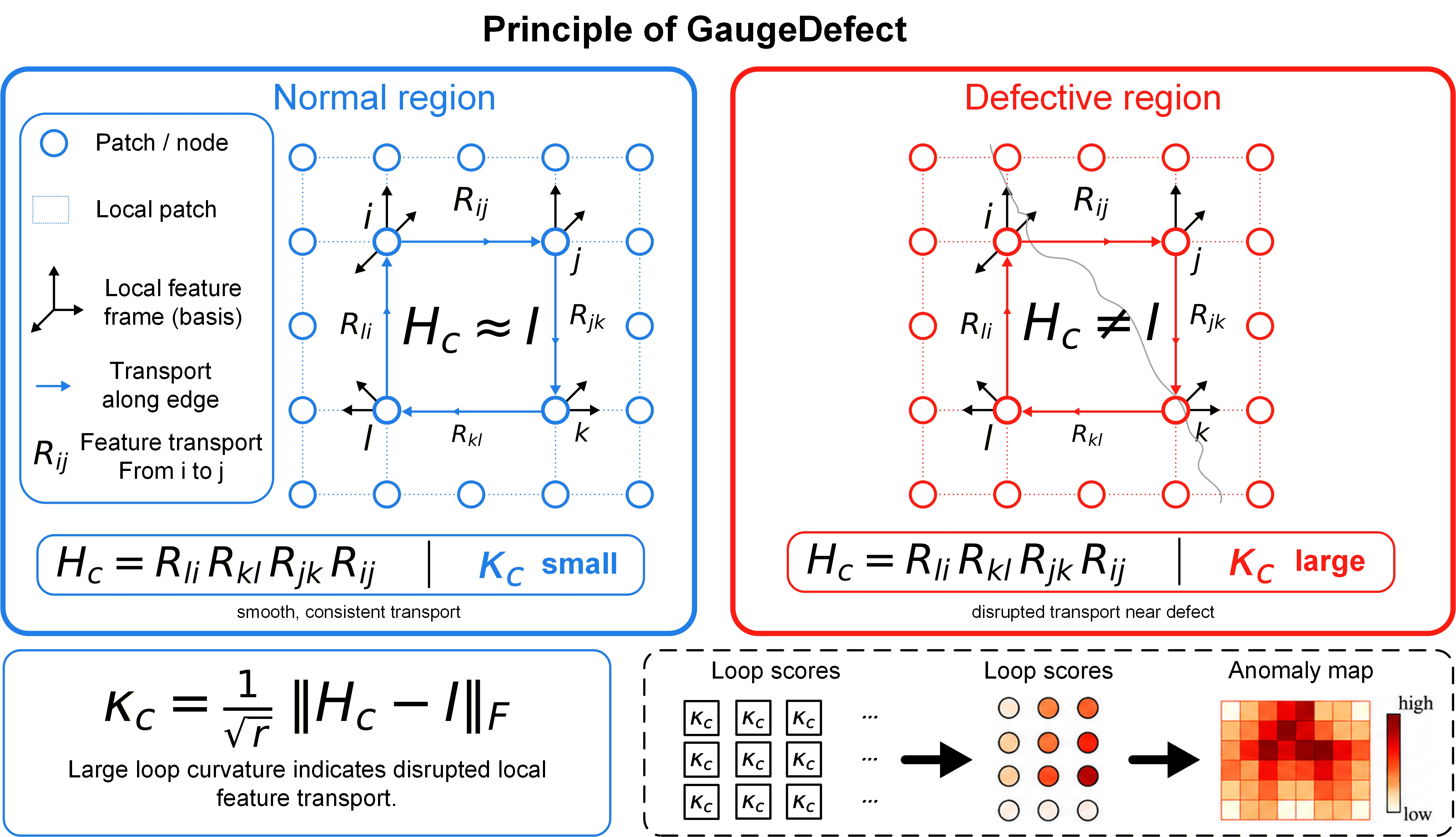}
	\caption{Principle of \method. In a normal region, local feature transport around a loop is consistent, so the holonomy is close to the identity and the curvature is small. In a defective region, the local feature field is disturbed, the same loop becomes inconsistent, and the curvature increases. Loop scores are then aggregated into node scores and finally into an anomaly map.}
	\label{fig:principle}
	\vspace{-3mm}
\end{figure}

\section{Method}

\subsection{Overview}

Given an image $x$, \method first extracts a spatial feature lattice from a pretrained backbone. Each lattice node corresponds to a deterministic image region: a patch token for a ViT-style backbone or a convolutional feature cell for a CNN-style backbone. Local feature frames are then estimated from neighboring nodes. Orthogonal transports are computed between adjacent frames, and feature-transport curvature is measured by the holonomy accumulated around small closed loops. The final anomaly map is obtained by calibrating loop curvature with normal training statistics and aggregating loop scores back to image space. The six stages in Fig.~\ref{fig:overview} follow this process. The principle of loop curvature is illustrated in Fig.~\ref{fig:principle}.

\subsection{Feature Lattice and Node Construction}

Let $\D_{\mathrm{N}}=\{x_m\}_{m=1}^{M}$ be the normal training set. For an input image $x$, a pretrained backbone $\Phi$ produces a spatial feature tensor
\begin{equation}
	F=\Phi(x)\in\R^{H\times W\times d}.
\end{equation}
We define a feature lattice $G=(V,E)$, where each node $i\in V$ corresponds to one spatial location on the $H\times W$ feature map, and carries a feature vector $f_i\in\R^d$. The definition of a node depends on the backbone but is deterministic in both cases.

\paragraph{ViT feature lattice.}
For a ViT-style backbone such as DINOv2, the image is resized to a fixed resolution divisible by the patch size $p$. The backbone outputs patch tokens arranged on a regular grid. Each node corresponds to one non-overlapping image patch,
\begin{equation}
	\Omega_i =
	[up,(u+1)p)\times[vp,(v+1)p),
\end{equation}
where $i=(u,v)$ indexes the patch-token position. The node feature $f_i$ is the corresponding patch-token embedding.

\paragraph{CNN feature lattice.}
For a CNN backbone such as ResNet50, we take an intermediate feature map with stride $s$. Each node corresponds to one spatial feature cell whose receptive field is centered at a fixed image location. If multiple feature levels are used, they are resized to a common lattice resolution and concatenated along the channel dimension. In the default setting, we use a single intermediate layer to avoid ambiguity in node definition.

\paragraph{Node normalization.}
Before building local frames, each node feature is $\ell_2$-normalized:
\begin{equation}
	f_i \leftarrow \frac{f_i}{\|f_i\|_2+\eps}.
\end{equation}
This keeps the subsequent frame construction focused on local feature directions rather than feature magnitude.

\paragraph{Spatial edges and elementary loops.}
Edges connect spatially adjacent nodes. By default, we use the four-neighbor grid:
\begin{equation}
	E=\{(i,j): j \text{ is the left, right, upper, or lower neighbor of } i\}.
\end{equation}
The default loop is the elementary plaquette formed by four adjacent lattice nodes. It is a $2\times2$ node loop on the feature lattice, traversed by four neighbor-wise transports. Boundary nodes that cannot participate in complete elementary plaquettes are excluded from loop scoring and receive scores only through interpolation from nearby valid nodes.

This construction makes node sampling explicit. The method does not depend on object masks, saliency maps, or hand-selected keypoints. The ablation study evaluates how the choice of backbone, feature layer, lattice resolution, and node density affects performance. This corresponds to stage 1 in Fig.~\ref{fig:overview}.

\subsection{Local Feature Frames}

For each node $i$, we define a local neighborhood $\mathcal{N}(i)$ as a $k\times k$ window on the feature lattice. The centered neighborhood matrix is
\begin{equation}
	X_i =
	\left[
	f_{j_1}-\bar f_i,\,
	f_{j_2}-\bar f_i,\,
	\ldots,\,
	f_{j_{|\mathcal{N}(i)|}}-\bar f_i
	\right]
	\in\R^{d\times |\mathcal{N}(i)|},
\end{equation}
where
\begin{equation}
	\bar f_i =
	\frac{1}{|\mathcal{N}(i)|}
	\sum_{j\in\mathcal{N}(i)} f_j.
\end{equation}

We compute $X_i=U_i\Sigma_iV_i^{\top}$ and take the top $r$ left singular vectors as the local frame: $Q_i = U_i[:,1:r]$.

Thus
\begin{equation}
	Q_i\in\mathrm{St}(d,r)
	=
	\{Q\in\R^{d\times r}:Q^{\top}Q=I_r\}.
\end{equation}

The rank $r$ must satisfy $r \leq |\mathcal{N}(i)|-1$, because the neighborhood features are centered. In our default setting, $k=3$ and $r=8$, so the local frame uses the maximal non-trivial rank of a centered $3\times3$ neighborhood.

The role of $Q_i$ is to describe the dominant local variation of features around node $i$. It is not an absolute coordinate system. Different orthogonal bases can represent the same local subspace, so the final score must be invariant to frame orientation. This corresponds to stage 2 in Fig.~\ref{fig:overview}.

\subsection{Orthogonal Feature Transport}

For each edge $(i,j)\in E$, we compute an orthogonal transport matrix $R_{ij}\in\mathrm{O}(r)$ by aligning the local frame at $j$ to the local frame at $i$:
\begin{equation}
	R_{ij}
	=
	\argmin_{R\in\mathrm{O}(r)}
	\|Q_i-Q_jR\|_F^2.
	\label{eq:procrustes}
\end{equation}
Let $Q_j^{\top}Q_i=U_{ij}\Sigma_{ij}V_{ij}^{\top}$. The orthogonal Procrustes solution is $R_{ij}=U_{ij}V_{ij}^{\top}$.

This transport aligns neighboring local coordinate systems while preserving lengths and angles in the $r$-dimensional feature-frame coordinates. The orthogonal constraint is important: a general linear map could absorb arbitrary local changes and weaken the meaning of loop inconsistency. This corresponds to stage 3 in Fig.~\ref{fig:overview}.

\subsection{Loop Holonomy and Feature-Transport Curvature}

Consider an elementary loop $c=(i,j,k,l)$ traversed as $i\rightarrow j\rightarrow k\rightarrow l\rightarrow i$. The accumulated transport around this loop is $H_c = R_{li}R_{kl}R_{jk}R_{ij}$.

If local feature transport is consistent, moving around the loop should return to nearly the same coordinate system: $H_c \approx I_r$. When the loop encloses a local disruption, the accumulated transport deviates from identity. We define the raw loop curvature as
\begin{equation}
	\kappa_c(x)
	=
	\frac{1}{\sqrt r}
	\|H_c-I_r\|_F.
	\label{eq:curvature}
\end{equation}

The normalization by $\sqrt r$ makes scores comparable across different frame ranks. Fig.~\ref{fig:principle} visualizes this principle: normal regions have small loop curvature, while defective regions lead to larger curvature because feature transport is disrupted. This corresponds to stage 4 in Fig.~\ref{fig:overview}.

\subsection{Gauge Invariance}

The local frame $Q_i$ is only defined up to an orthogonal change of basis. If $Q_i' = Q_iG_i, G_i\in\mathrm{O}(r)$, then $Q_i$ and $Q_i'$ represent the same local subspace. A valid curvature score should not depend on this arbitrary choice.

\begin{proposition}
	Under arbitrary local frame rotations $Q_i'=Q_iG_i$, the loop holonomy transforms by orthogonal conjugation: $H_c' = G_i^{\top}H_cG_i$. Therefore, the curvature score in Eq.~\eqref{eq:curvature} is invariant.
\end{proposition}

\begin{proof}
	Under the frame change, the edge transport becomes $R_{ij}'=G_j^{\top}R_{ij}G_i$. For the loop $c=(i,j,k,l)$,
	\begin{align}
		H_c'
		&=
		R_{li}'R_{kl}'R_{jk}'R_{ij}' \nonumber\\
		&=
		(G_i^{\top}R_{li}G_l)
		(G_l^{\top}R_{kl}G_k)
		(G_k^{\top}R_{jk}G_j)
		(G_j^{\top}R_{ij}G_i) \nonumber\\
		&=
		G_i^{\top}
		(R_{li}R_{kl}R_{jk}R_{ij})
		G_i \nonumber\\
		&=
		G_i^{\top}H_cG_i.
	\end{align}
	Since the Frobenius norm is invariant under orthogonal conjugation,
	\begin{equation}
		\|H_c'-I_r\|_F
		=
		\|G_i^{\top}(H_c-I_r)G_i\|_F
		=
		\|H_c-I_r\|_F.
	\end{equation}
	Thus the curvature score is independent of local frame orientation.
\end{proof}

\subsection{Normal Curvature Calibration}

Raw curvature is not used directly, because normal surfaces may naturally have stable non-zero curvature in feature space. For example, a curved object or a textured material may produce regular local frame rotations. We therefore estimate normal curvature statistics from the normal training set.

For each loop location $c$, we compute
\begin{equation}
	\mu_c
	=
	\frac{1}{M}
	\sum_{m=1}^{M}\kappa_c(x_m),
\end{equation}
and
\begin{equation}
	\sigma_c
	=
	\sqrt{
		\frac{1}{M-1}
		\sum_{m=1}^{M}
		\left(\kappa_c(x_m)-\mu_c\right)^2
	}.
\end{equation}
For a test image $x$, the calibrated loop score is
\begin{equation}
	s_c(x)
	=
	\frac{\kappa_c(x)-\mu_c}{\sigma_c+\eps}.
	\label{eq:calibration}
\end{equation}
We use the positive part: $\hat{s}_c(x)=\max(0,s_c(x))$.

This calibration makes the method respond to abnormal deviations from normal transport behavior rather than to non-zero curvature itself. This corresponds to stage 5 in Fig.~\ref{fig:overview}.

\subsection{Dense Anomaly Map}

Each loop score is distributed to the four nodes of the loop. Let $\mathcal{C}(i)$ be the set of valid loops incident to node $i$. The node-level anomaly score is
\begin{equation}
	a_i(x)
	=
	\frac{1}{|\mathcal{C}(i)|}
	\sum_{c\in\mathcal{C}(i)}
	\hat{s}_c(x).
\end{equation}
The dense anomaly map is then obtained by bilinear upsampling:
\begin{equation}
	A(x)
	=
	\mathrm{Upsample}\left(\{a_i(x)\}_{i\in V}\right).
\end{equation}
For image-level anomaly detection, we use a top-$q$ average:
\begin{equation}
	S_{\mathrm{img}}(x)
	=
	\frac{1}{|\Omega_q|}
	\sum_{u\in\Omega_q} A_u(x),
\end{equation}
where $\Omega_q$ denotes the set of pixels with the top $q$ percent anomaly scores. This corresponds to stage 6 in Fig.~\ref{fig:overview}.

\subsection{Computational Complexity}

Let $N=HW$ be the number of lattice nodes, $m=k^2$ the neighborhood size, and $r$ the frame rank. Local frame construction requires truncated SVD on $d\times m$ matrices, edge transport requires SVD on $r\times r$ matrices, and loop curvature requires multiplying four $r\times r$ matrices. Since $m$ and $r$ are small, all computations are local and parallelizable. The loop-curvature computation has complexity $O(|\mathcal{C}|r^3)$, where $|\mathcal{C}|$ is the number of valid loops.

\begin{table}[t]
	\centering
	\caption{Main results on MVTec AD.}
	\label{tab:main_mvtec}
	\resizebox{0.85\linewidth}{!}{
		\begin{tabular}{lccccc}
			\toprule
			Method & I-AUROC $\uparrow$ & I-AP $\uparrow$ & P-AUROC $\uparrow$ & P-AP $\uparrow$ & AUPRO $\uparrow$ \\
			\midrule
			PatchCore~\cite{Roth2021TowardsTR} & 99.1 & 99.2 & 98.1 & 55.3 & 93.2 \\
			SimpleNet~\cite{Liu2023SimpleNetAS} & 99.4 & 99.3 & 98.0 & 58.2 & 89.9 \\
			EfficientAD~\cite{Batzner2023EfficientADAV} & 99.1 & 99.3 & 97.8 & 60.2 & 93.0 \\
			RealNet~\cite{Zhang2024RealNetAF} & 99.2 & 99.2 & 98.9 & 62.8 & 92.8 \\
			GLASS~\cite{Chen2024AUA} & 99.0 & 99.2 & 98.2 & 63.7 & 93.5 \\
			GeneralAD~\cite{Strater2024GeneralADAD} & 99.4 & 99.5 & 98.1 & 65.1 & 94.0 \\
			INP-Former~\cite{Luo2025ExploringIN} & 99.5 & 99.6 & 98.4 & 68.4 & 95.2 \\
			Dinomaly~\cite{Guo2024DinomalyTL} & 99.3 & 99.2 & 98.0 & 68.5 & 94.3 \\
			ResNet50-NN~\cite{He2015DeepRL} & 97.9 & 98.1 & 96.4 & 53.4 & 89.5 \\
			DINOv2-NN~\cite{Oquab2023DINOv2LR} & 99.4 & 99.5 & 98.2 & 67.2 & 94.5 \\
			\rowcolor{gray!15}
			\method & \textbf{99.7}$\pm$0.1 & 99.7$\pm$0.1 & \textbf{98.8}$\pm$0.1 & \textbf{76.4}$\pm$0.5 & \textbf{96.5}$\pm$0.3 \\
			\bottomrule
	\end{tabular}}
	\vspace{-3mm}
\end{table}

\section{Experiments}

\subsection{Experimental Setup}

We evaluate \method on three industrial anomaly localization benchmarks: MVTec AD~\cite{Bergmann2019MVTecA}, VisA~\cite{Zou2022SPottheDifferenceSP}, and Real-IAD~\cite{Wang2024RealIADAR}. MVTec AD is used as the main benchmark because it is the standard reference for industrial anomaly localization. VisA tests generalization to more object categories and defect types. Real-IAD evaluates behavior in a larger real-world multi-view setting.

\subsection{Feature Lattice Implementation}

The feature lattice is the starting point of \method, so we make its construction explicit. DINOv2~\cite{Oquab2023DINOv2LR} lattice: The input image is resized to a fixed resolution divisible by the ViT patch size. Patch tokens are reshaped into a spatial grid and used as lattice nodes. Each node corresponds to one image patch and carries the corresponding token embedding. ResNet50 lattice: We also evaluate a ResNet50 backbone~\cite{He2015DeepRL}. For ResNet50, nodes are taken from an intermediate convolutional feature map. Each node corresponds to one spatial feature cell. This setting is used to test whether the proposed curvature score depends on transformer tokens or can also be computed on conventional CNN features. Default setting: The default configuration uses dense node sampling on the feature lattice, four-neighbor spatial edges, $3\times3$ local neighborhoods, frame rank $r=8$, and elementary $2\times2$ node loops. Boundary nodes that cannot form complete loops are excluded from loop computation and filled through interpolation during upsampling. Implementation details: All experiments are conducted on a single NVIDIA A100 GPU with 80GB memory. Unless otherwise specified, all hyperparameters use the default configuration described above. The default GaugeDefect setting uses the DINOv2 patch-token lattice, the middle feature layer, dense node sampling, frame rank $r=8$, and the elementary $2\times2$ node loop.

\subsection{Evaluation Metrics and Compared Methods}

For image-level anomaly detection, we report image AUROC (I-AUROC) and image average precision (I-AP). For pixel-level anomaly localization, we report pixel AUROC (P-AUROC), pixel average precision (P-AP), and AUPRO. Pixel-level AP and AUPRO are particularly important for small or thin defects, because pixel-level class imbalance can make AUROC less sensitive to localization quality. We compare \method with representative industrial anomaly localization methods, including PatchCore~\cite{Roth2021TowardsTR}, SimpleNet~\cite{Liu2023SimpleNetAS}, EfficientAD~\cite{Batzner2023EfficientADAV}, RealNet~\cite{Zhang2024RealNetAF}, GLASS~\cite{Chen2024AUA}, GeneralAD~\cite{Strater2024GeneralADAD}, INP-Former~\cite{Luo2025ExploringIN}, and Dinomaly~\cite{Guo2024DinomalyTL}. These baselines cover memory-based feature matching, feature-space anomaly generation, efficient distillation, synthetic anomaly modeling, cross-domain anomaly detection, prototype-based detection, and recent multi-class unsupervised anomaly detection. To distinguish the contribution of feature-transport curvature from the strength of pretrained features, DINOv2-NN and ResNet50-NN in Table~\ref{tab:main_mvtec} refer to DINOv2- and ResNet50-based variants that replace loop-holonomy curvature with pointwise nearest-neighbor feature distance. These two baselines use the same feature-lattice construction as \method.

\subsection{Main Quantitative Results and Qualitative Comparison}
\begin{wraptable}{r}{0.6\columnwidth} 
	\centering
	\vspace{-6mm}
	\caption{Cross-dataset AUPRO summary.}
	\label{tab:cross_dataset}
	\begin{tabular}{lccc}
		\toprule
		Method & MVTec AD & VisA & Real-IAD \\
		\midrule
		PatchCore~\cite{Roth2021TowardsTR} & 93.2 & 91.2 & 91.0 \\
		EfficientAD~\cite{Batzner2023EfficientADAV} & 93.0 & 92.2 & 89.2 \\
		SimpleNet~\cite{Liu2023SimpleNetAS} & 89.9 & 90.1 & 84.1 \\
		RealNet~\cite{Zhang2024RealNetAF} & 92.8 & 92.6 & 90.8 \\
		Dinomaly~\cite{Guo2024DinomalyTL} & 94.3 & 93.7 & 93.3 \\
		\rowcolor{gray!15}
		\method & \textbf{96.5} & \textbf{95.1} & \textbf{94.4} \\
		\bottomrule
	\end{tabular}
	\vspace{-3mm}
\end{wraptable}
Table~\ref{tab:main_mvtec} reports the main results on MVTec AD. While image-level metrics are already close to saturation on MVTec AD, \method shows clearer gains on localization-sensitive metrics, especially P-AP and AUPRO. Compared with DINOv2-NN, the improvement indicates that the gain is not merely due to stronger pretrained features, but comes from the proposed feature-transport curvature. Table~\ref{tab:cross_dataset} summarizes the cross-dataset AUPRO results. The consistent gains on VisA and Real-IAD suggest that the proposed curvature signal is not restricted to the standard MVTec AD setting, and remains useful under more diverse categories and larger real-world inspection scenarios. Fig.~\ref{fig:compare} compares PatchCore, EfficientAD, SimpleNet, Dinomaly, and \method on representative examples. The visual comparison focuses on thin scratches, small dents, and disrupted local patterns, where the local appearance can be ambiguous but the neighborhood relation is disturbed. To highlight discriminative regions and reduce background distraction, normal regions are masked in the visualization, and low-confidence responses with normalized scores below $0.4$ are omitted.

\begin{figure}[t]
	\centering
	\includegraphics[width=0.85\linewidth]{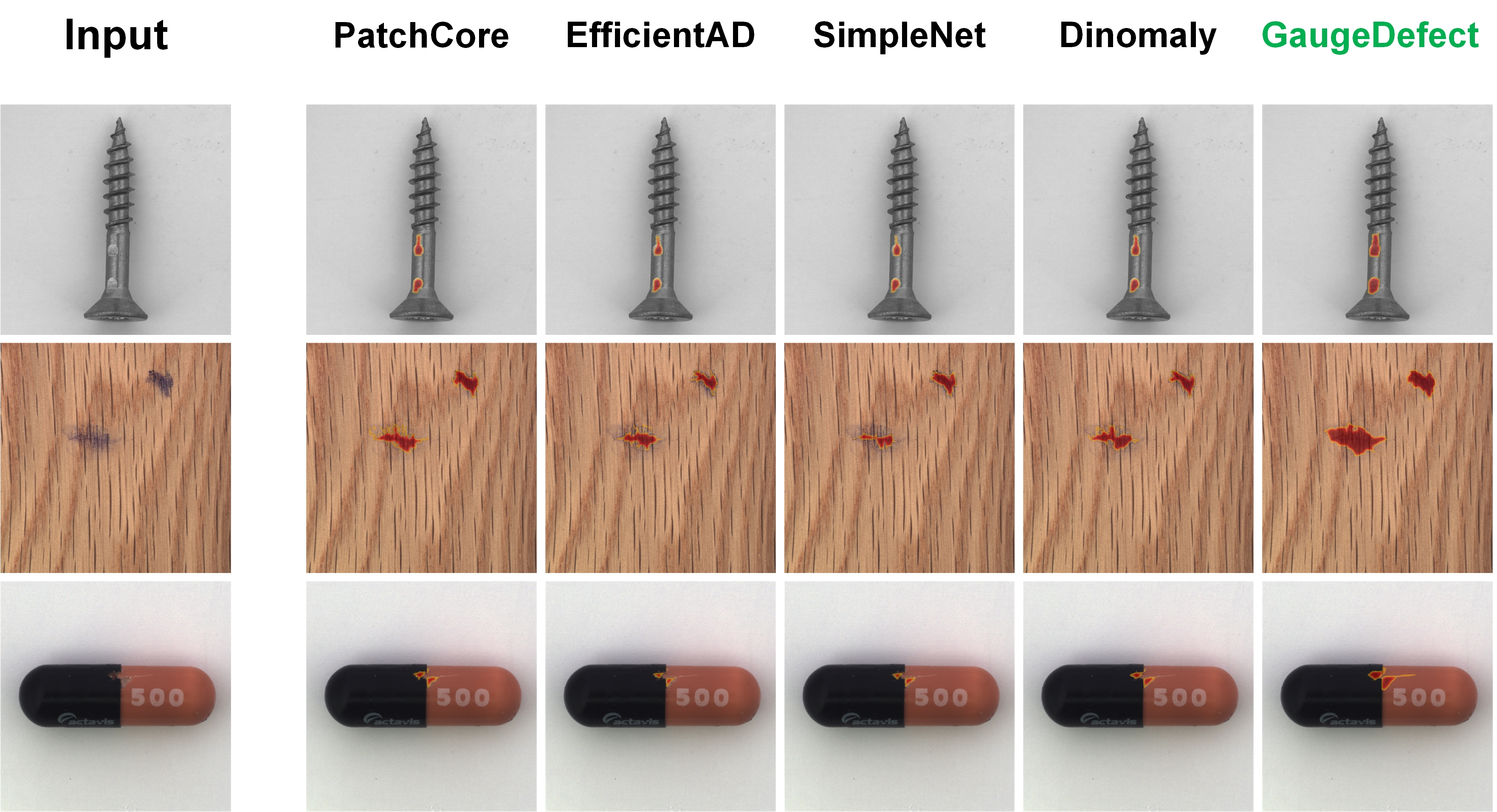}
	\caption{Qualitative comparison of anomaly localization results. \method tends to provide more precise responses on subtle surface disruptions such as thin scratches, small dents, and broken local patterns.}
	\label{fig:compare}
	\vspace{-3mm}
\end{figure}

\begin{table}[t]
	\centering
	\caption{Ablation on feature-lattice node construction. Except for the ResNet50 row, all variants are based on the default DINOv2 configuration}
	\label{tab:node_ablation}
	\begin{tabular}{lccc}
		\toprule
		Node construction & P-AUROC $\uparrow$ & P-AP $\uparrow$ & AUPRO $\uparrow$ \\
		\midrule
		DINOv2 patch-token lattice & 98.8 & 76.4 & 96.5 \\
		ResNet50 convolutional lattice & 97.2 & 59.6 & 91.8 \\
		Early feature layer & 97.9 & 64.8 & 93.4 \\
		Middle feature layer & 98.8 & 76.4 & 96.5 \\
		Late feature layer & 98.3 & 70.6 & 94.7 \\
		Dense nodes $(\rho=1)$ & 98.8 & 76.4 & 96.5 \\
		Subsampled nodes $(\rho=2)$ & 98.1 & 69.7 & 94.2 \\
		Subsampled nodes $(\rho=4)$ & 97.3 & 61.5 & 91.6 \\
		\bottomrule
	\end{tabular}
	\vspace{-3mm}
\end{table}

\subsection{Ablation on Node Construction}

Node construction is a core part of \method, since all subsequent frames, transports, and loops are defined on the feature lattice. Table~\ref{tab:node_ablation} evaluates several node-construction choices. The default GaugeDefect configuration is based on the DINOv2 patch-token lattice with the middle feature layer and dense node sampling. Backbone type: The DINOv2 patch-token lattice clearly outperforms the ResNet50 convolutional lattice, indicating that stable self-supervised patch features are more suitable for estimating local feature frames and transport curvature. This also supports our assumption that robust pretrained patch tokens provide a stronger basis for local frame construction. Feature layer: For DINOv2, we compare different transformer layers. For ResNet50, we compare intermediate convolutional stages. Early layers may preserve local texture but lack semantic stability, while later layers may be more robust but spatially coarser. Lattice resolution: We evaluate different feature resolutions by resizing the feature map or selecting different backbone feature resolutions. A finer lattice can capture thin defects, while a coarser lattice can be more stable but may blur small anomalies. Node density: Subsampling nodes reduces localization accuracy, with a larger drop at $\rho=4$. This confirms that dense lattice coverage is important for thin and small defects.

\begin{figure}[t]
	\centering
	\includegraphics[width=\linewidth]{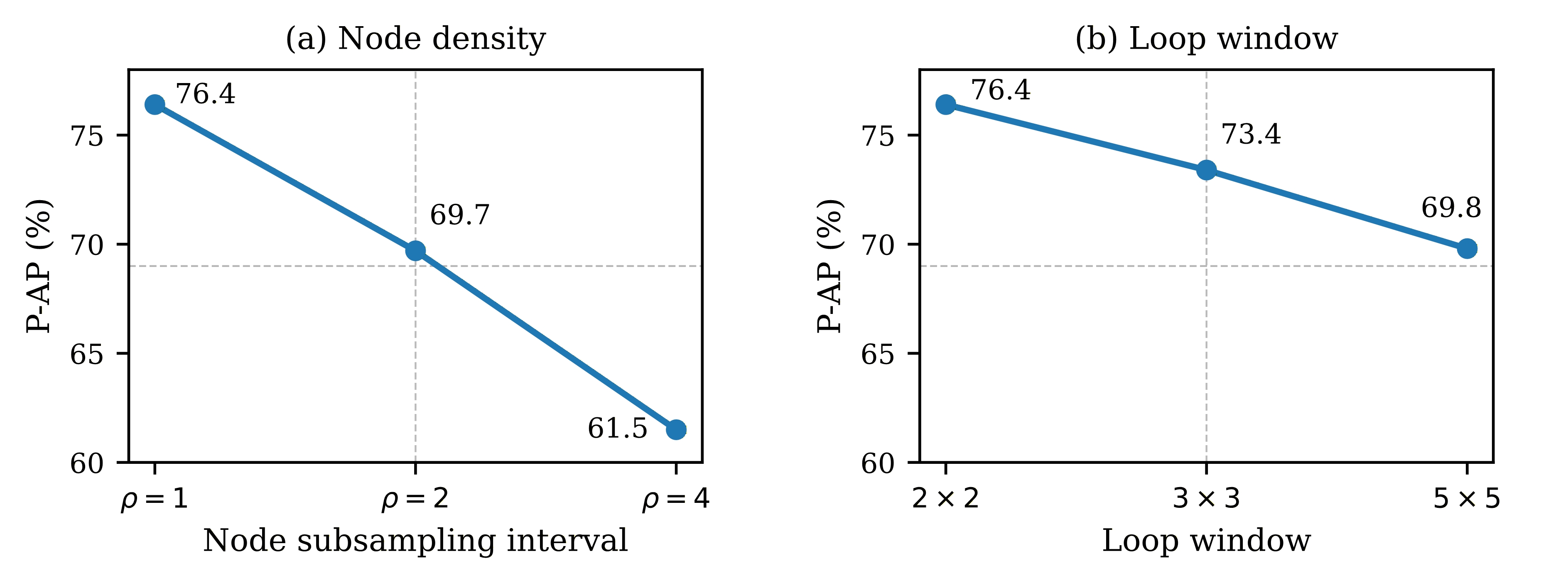}
	\caption{Sensitivity analysis on spatial support.}
	\label{fig:sensitivity}
	\vspace{-3mm}
\end{figure}

\subsection{Ablation on Curvature Components}

We further ablate the geometric components of \method under the default DINOv2 configuration. The results in Table~\ref{tab:curvature_ablation} show that the full model consistently performs best, especially on P-AP and AUPRO. Curvature vs. pointwise scores: Replacing feature-transport curvature with pointwise nearest-neighbor feature distance leads to a clear drop in P-AP. This indicates that the gain of \method does not only come from the pretrained feature lattice, but also from the closed-loop transport consistency. Loop holonomy: Using only neighboring frame discrepancy is weaker than the full holonomy score. This suggests that closed-loop composition captures a stronger relational signal than pairwise frame differences. Calibration and transport: Removing normal curvature calibration reduces performance, showing that normal surfaces can have stable non-zero transport curvature. The fixed-transport variant $R_{ij}=I_r$ drops more clearly, confirming that explicit orthogonal frame alignment is necessary. Frame rank and loop window: With the default $3\times3$ local neighborhood, we evaluate $r\in\{2,4,8\}$, where $r=8$ gives the best result. We also vary the loop window from $2\times2$ to $3\times3$ and $5\times5$. The default $2\times2$ loop is the elementary node loop and performs best; larger loop windows capture broader context but dilute thin scratches and small local disruptions. Fig.~\ref{fig:sensitivity} shows that dense node coverage and small loop windows are important for capturing local transport inconsistency in subtle defects. Dense nodes and the default $2\times2$ loop window give the best P-AP.

\begin{table}[t]
	\centering
	\caption{Ablation on feature-transport curvature. All variants are evaluated under the default DINOv2 configuration.}
	\label{tab:curvature_ablation}
	\begin{tabular}{lccc}
		\toprule
		Variant & P-AUROC $\uparrow$ & P-AP $\uparrow$ & AUPRO $\uparrow$ \\
		\midrule
		\rowcolor{gray!15}
		Full \method & 98.8 & 76.4 & 96.5 \\
		Feature distance only & 98.2 & 67.2 & 94.6 \\
		Frame discrepancy only & 98.0 & 65.8 & 94.0 \\
		Raw curvature without calibration & 98.1 & 68.9 & 94.3 \\
		Fixed transport $R_{ij}=I_r$ & 97.6 & 62.7 & 92.8 \\
		Rank $r=2$ & 98.1 & 68.4 & 94.3 \\
		Rank $r=4$ & 98.5 & 72.8 & 95.6 \\
		Rank $r=8$ & 98.8 & 76.4 & 96.5 \\
		Loop window $2\times2$ (default) & 98.8 & 76.4 & 96.5 \\
		Loop window $3\times3$ & 98.5 & 73.4 & 95.4 \\
		Loop window $5\times5$ & 98.1 & 69.8 & 94.2 \\
		\bottomrule
	\end{tabular}
	\vspace{-3mm}
\end{table}

\subsection{Additional Protocol and Efficiency}

\begin{wrapfigure}{r}{0.5\textwidth}
	\centering
	\vspace{-6mm}
	\includegraphics[width=\linewidth]{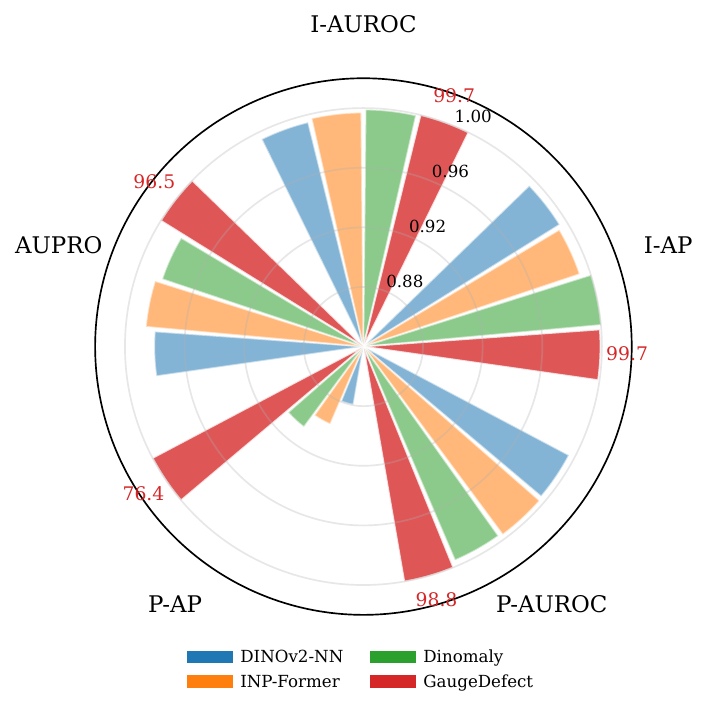}
	\caption{Metric profile on MVTec AD with per-axis best-score normalization.}
	\label{fig:petal_metrics}
	\vspace{-5mm}
\end{wrapfigure}

\method follows the standard single-class, per-category protocol. All datasets and categories use the same dense lattice, $3{\times}3$ neighborhood, rank $r{=}8$, and $2{\times}2$ loop; only the normal statistics $(\mu_c,\sigma_c)$ are category-specific. To separate the contribution of transport curvature from the pretrained backbone, we compare \method with the exact-lattice baseline DINOv2-NN, which uses the same DINOv2-B/14 lattice but replaces curvature with nearest-neighbor feature distance. Fig.~\ref{fig:petal_metrics} summarizes the metric profile on MVTec AD using per-axis best-score normalization. The gain of \method is concentrated on localization quality, especially P-AP and AUPRO. Table~\ref{tab:additional_eval} reports the added image-level results and end-to-end runtime. Runtime is measured on MVTec AD with a single A100 GPU, $512{\times}512$ inputs, batch size 1, and excluding disk I/O. Although local SVD and transport increase latency, \method still runs at 15.4 FPS while providing the strongest localization accuracy.

\begin{table}[t]
	\centering
	\caption{Image-level results and runtime.}
	\label{tab:additional_eval}
	\scriptsize
	\setlength{\tabcolsep}{2.5pt}
	\renewcommand{\arraystretch}{0.90}
	\resizebox{0.95\linewidth}{!}{
		\begin{tabular}{lccccc}
			\toprule
			Entry & I-AUROC $\uparrow$ & I-AP $\uparrow$
			& ms $\downarrow$ & FPS $\uparrow$ & GPU/CPU GB $\downarrow$ \\
			\midrule
			\multicolumn{6}{l}{\textit{Image-level detection}} \\
			VisA & 98.8 & 99.1 & -- & -- & -- \\
			Real-IAD & 97.6 & 98.2 & -- & -- & -- \\
			\midrule
			\multicolumn{6}{l}{\textit{MVTec AD runtime}} \\
			Dinomaly (DINOv2-B/14) & -- & -- & 36.8 & 27.2 & 3.6/4.0 \\
			INP-Former (DINOv2-B/14) & -- & -- & 52.4 & 19.1 & 5.0/4.8 \\
			D2Rec (ViT-B) & -- & -- & 45.7 & 21.9 & 4.4/4.6 \\
			\method (DINOv2-B/14) & -- & -- & 64.9 & 15.4 & 5.6/4.5 \\
			\bottomrule
	\end{tabular}}
	\vspace{-3mm}
\end{table}

\section{Conclusion}

\method supports more reliable and efficient industrial surface inspection. The main idea is to measure whether local feature frames remain consistent under transport around small closed loops. This leads to a gauge-invariant curvature score that captures local inconsistency in the feature field. The method complements existing detectors with a geometric signal for defects whose local appearance remains plausible. In practice, this often leads to more precise localization and clearer responses near the support of subtle defects.

\subsubsection*{Competing Interests}
The authors declare no competing interests relevant to this work.

%
%
%
\bibliographystyle{splncs04}
\bibliography{bibliography}

\end{document}